\documentclass[lettersize,journal]{IEEEtran}
\IEEEoverridecommandlockouts       
\usepackage{geometry}
\usepackage[style=ieee,backend=biber]{biblatex}
\usepackage{graphics} 
\usepackage{epsfig} 
\usepackage{setspace}
\usepackage{color}
\usepackage{booktabs}
\usepackage{times} 
\usepackage{threeparttable}
\usepackage{amsmath} 
\usepackage{amsthm}
\newtheorem{lemma}{Lemma}[section]
\newtheorem{theorem}{Theorem}[section]
\usepackage{amssymb}  
\usepackage{subfigure}
\usepackage{flushend}
\usepackage{algorithm} 
\usepackage{algorithmic} 
\usepackage{multirow} 
\usepackage{pifont}
\definecolor{darkblue}{rgb}{0.8,0,0.0}

\newtheorem{myDef}{Definition}
\newtheorem{myAum}{Assumption}
\newtheorem{myPro}{Proposition}

\begin{document} 
  \title{\LARGE \bf
 A Flexible and Resilient Formation Approach based on Hierarchical Reorganization}
   \author{~Yuzhu~Li, ~Wei~Dong$^*$
\thanks{
The authors are with the State Key Laboratory of Mechanical System
 	and Vibration, School of Mechanical Engineering, Shanghai Jiao Tong University, Shanghai 200240, China (e-mail: \{yuzhu\_0222, dr.dongwei\}@sjtu.edu.cn).}.
  }
 \markboth{Journal of \LaTeX\ Class Files,~Vol.~X, No.~X, X~XXXX}{Shell \MakeLowercase{\textit{et al.}}: Bare Demo of IEEEtran.cls  for IEEE Journals}

\maketitle 
 \begin{abstract}  
   Conventional formation methods typically rely on fixed hierarchical structures, such as predetermined leaders or predefined formation shapes. These rigid hierarchies can render formations cumbersome and inflexible in complex environments, leading to potential failure if any leader loses connectivity. To address these limitations, this paper introduces a reconfigurable affine formation that enhances both flexibility and resilience through hierarchical reorganization. The paper first elucidates the critical role of hierarchical reorganization, conceptualizing this process as involving role reallocation and dynamic changes in topological structures. To further investigate the conditions necessary for hierarchical reorganization, a reconfigurable hierarchical formation is developed based on graph theory, with its feasibility rigorously demonstrated. In conjunction with role transitions, a power-centric topology switching mechanism grounded in formation consensus convergence is proposed, ensuring coordinated resilience within the formation. Finally, simulations and experiments validate the performance of the proposed method. The aerial formations successfully performed multiple hierarchical reorganizations in both three-dimensional and two-dimensional spaces. Even in the event of a single leader's failure, the formation maintained stable flight through hierarchical reorganization. This rapid adaptability enables the robotic formations to execute complex tasks, including sharp turns and navigating through forests at speeds up to 1.9 m/s.
 \end{abstract}
 
 \begin{IEEEkeywords}
 Hierarchical Reorganization, Reconfigurable hierarchical formation, Leader Selection, Swarm 
 \end{IEEEkeywords}
 
 \section{Introduction}
   Aerial robots have garnered growing scholarly interest in recent years due to their vast potential for application in collaborative tasks, including exploration\cite{Hu2021, Zhou2022}, inspection\cite{Abdelkader2021}, fire rescue\cite{Merino2012}. Swarm formations, in particular, exhibit remarkable performance in self-organization, survivability, and collaborative task execution\cite{Dong2015}. In complex environments, adopting variable formations during flight becomes crucial for maintaining optimal mobility and flexibility\cite{Yu2022}.
   
   The architectural frameworks facilitating formation reorganization can principally be categorized into democratic and autocratic. In democratic formations, the entire group is generally regarded as a unified entity, each planning their actions based on collective information and influencing the overall formation. For instance, in \cite{Zhou2018}, the formation is conceptualized as a virtual rigid body, which allows for versatile transformations of various formation patterns by adjusting the relative orientations and positions of individual aerial robots with a virtual center. This approach facilitates the execution of complex tasks, such as navigating through cluttered buildings, dense forests, or disaster-stricken sites. However, this method relies on a virtual structure to govern the formation's pattern, demanding real-time computation of transfer vectors for each individual within the formation. As the scale of aerial robots increases, the complexity of democratic reconfiguration correspondingly escalates, thereby diminishing the formation's flexibility and robustness.
   
   Conversely, autocratic formation reorganization relies on a subset of agents within the formation. By hierarchically distributing the power among different individuals, certain members exert influence over the overall formation pattern, facilitating formation reorganization. For example, Ref. \cite{Zhao2018} introduces an affine formation method within a leader-follower framework. This method achieves comprehensive formation changes through the strategic positioning of the leader, enabling seamless continuous transitions, including translation, rotation, contraction, and deformation. However, the autocratic formations' leaders remain unchanged, forming a unidirectional affine transformation. This configuration may pose challenges when the swarm needs to execute sharp turns or large-angle maneuvers, potentially limiting its flexibility. Notably, the entire formation with fixed hierarchy is at risk of collapse if any single leader becomes inoperative.
   
   Given the abovementioned challenges, our approach draws inspiration from the swerving and hierarchical formation movements observed in starling clusters, as explored in \cite{Mateo2017, Attanasi2014}. Starling flocks typically feature one or multiple leaders, with the formation's internal framework exhibiting a structure between democratic and autocratic paradigms. Building on this, the cluster dynamically adjusts its movement patterns, leading to hierarchical reorganization within the formation. In light of this natural phenomenon, we introduce an reconfigurable hierarchical formation(RHF), leveraging hierarchical reorganizations to enhance the flexibility and resilience of formation.
   
   First, we elucidate the definition of hierarchical reorganization for a formation. Building on this, the paper defines a reconfigurable hierarchical formation based on affine formation theory and analyzes the configuration requirements necessary for achieving hierarchical reorganization, providing rigorous proof. Subsequently, to enhance the internal synchronization of dynamic formations, a power-centric topology switching mechanism is proposed based on formation consensus convergence. This mechanism adapts to the dynamic reorganization of intra-formation connections, ensuring coordinated resilience of the formation.
   
   The main contributions of the paper are as follows:
   \begin{enumerate}
    \item A novel RHF strategy is proposed to efficiently improve the formation's flexibility and resilience.  
    \item The necessary conditions and theory analysis for hierarchical reorganization are proposed, with their feasibility and necessity rigorously demonstrated.
    \item A power-centric topology switching algorithm is put forward, aimed at mitigating disturbances caused by followers to further enhance the formation's internal synchrony.
    \item Simulation and Experiments are carried out to verify the performance of the RHF strategy.

   \end{enumerate} 
   
   The remainder of this paper is organized as follows. Related works are introduced in Section \MakeUppercase{\romannumeral2}. Preliminaries and problem formulation are given in Section \MakeUppercase{\romannumeral3}. The specific methodology is described in Section \MakeUppercase{\romannumeral4}, which contains specific analysis of RHF and a power-centric topology switch strategy. Verification is carried out in Section \MakeUppercase{\romannumeral5}, while Section \MakeUppercase{\romannumeral6} concludes the paper.
   
   \section{Related Works}
   The related works are segmented into two sections. The first section explores hierarchical affine formation, including the development of affine formations and the current state of dynamic formation development. The second section critically evaluates the role of topological switching, emphasizing its critical importance in hierarchical formation.
   
   \subsection{Hierarchical Affine Formation}
   Hierarchical formation control, with its advantage of addressing capability disparities among individuals, has become one of the primary methods in formation control\cite{Chen2015}. Numerous researchers have investigated formation control problems using leader-follower strategies \cite{Cao2013, Oh2015}. They have identified various formation structures, such as single-leader formation\cite{Kang2014, Babazadeh2018, Dai2022}, multi-leader formation\cite{Fang2022, Nomura2021}, and virtual leader formation\cite{Porfiri2007, Droge2015}. Multi-leader structures, which allow for affine formation through adjusting leaders' relative positions, have gained particular interest compared to single-leader and virtual-leader formation. 
   
    To achieve affine transformations, Ref.\cite{Lin2016} explores the sufficient and necessary conditions for realizing affine formations based on graph theory. Building on this, Ref.\cite{Zhao2018} proposes a control method for leader-follower formations based on affine transformation, demonstrating through simulations that the following error of followers converges exponentially. Furthering this approach, Ref.\cite{Li2021} introduces a hierarchical affine control algorithm to achieve formation control under conditions of non-global information, thereby enhancing the robustness of the control strategy. However, these control algorithms are all based on static and predetermined structures, which may be unresponsive and cumbersome, making it difficult for the formation to perform large maneuvers while maintaining formation in complex environments.

    Based on this, researchers have found that in practical situations, dynamic leader selection has been shown to improve the overall performance of a robot team compared to a static leader \cite{Franchi2011}. Current dynamic leader selection involves changes in leader roles during the task due to various factors. On the one hand, this includes human or human base station influence over the robotic system. For instance, Ref.\cite{Misra2021, AbdelMalek2022, Ganesan2020} improve communication quality with the base station by online leader selection in the swarm, and Ref.\cite{Saeidi2017} proposes using human-robot trust as a dynamic criterion for leader selection, reducing task completion time and formation errors compared to non-leader switching strategies. However, these methods are constrained by the quality of human-machine communication and the accuracy of human judgment on the swarm's current state, making them unsuitable for autonomous systems.

    Considering the autonomy of the swarm, another aspect of dynamic leader selection is the influence of the external environment on the robotic swarm. For example, Ref.\cite{Tavares2020} suggests a method for reselecting leaders in case of leader failure in environments with dense obstacles. To address the leader trapping issue caused by different probabilities of encountering obstacles for different roles within the formation, Ref.\cite{Li2017} introduces an emotion-based model for leader selection. This model allows the team to autonomously reselect a leader when trapped and continue moving towards the goal. However, while this method can increase the probability of individual escape, it does not guarantee the successful escape of the entire swarm. 
    
   \subsection{Topology Switching for Hierarchical Formation}
    Research has indicated that internal topology switching within a formation can enhance multiple capabilities of the formation. For example, Ref.\cite{Yu2022a} enhances the resilience of a swarm by designing internal topology transitions; Ref.\cite{Zhao2023} proposes an optimal rigidity graph-based topology optimization algorithm to reduce the communication complexity of formations, thereby extending the network's lifespan; Ref.\cite{Yu2020a} designs smooth transitions in communication to ensure security during topological changes. The introduction of dynamic topology switching into formation control has added new variables \cite{Mesbahi2010, OlfatiSaber2007}, making this area a subject of extensive research.

    Considering the control problems of formations under topology switching, Ref.\cite{Dong2017} proposes the necessary and sufficient conditions for achieving time-varying formation tracking under topology transitions, and the system's stability has been verified experimentally. Building on this, Ref.\cite{Zou2018} presents a distributed control algorithm using neighboring positions and velocities during the topology transition process when weak connections exist in the communication topology among individuals in the swarm. Considering different formation structures, the problem of hierarchical formation control with topology switching has been further explored.

    To investigate the impact of topology switching on the control of hierarchical formation structures, Ref.\cite{Liu2008} conducts a controllability analysis of hierarchical formations under topology switching. Based on this, Refs.\cite{Wang2012, Soni2021} study the control issues of leaders and followers during topology transitions. Furthermore, Refs.\cite{Lin2016, Zhao2018} explore the relationship between hierarchical control issues in affine formation and topological structures. Considering the introduction of hierarchical reorganization strategies in these problems, Ref.\cite{Xiao2022} optimizes the formation of topological structures to reduce communication costs while improving convergence speed. However, no method has yet explored the impact of topology switching on formation control during hierarchical reorganizations.
  
 \section{Preliminaries and Problem Statement} 
  \subsection{Basic Graph and Formation Theory}
   Consider a group of $n$ mobile agents in ${\mathbb{R}}^d$ where $d > 2$ and $n \geq d + 1$. Let $\boldsymbol{p}_i \in \mathbb{R}^d$ be the position of agent $i$ and $\boldsymbol{p} = [\boldsymbol{p}_1^T,...,\boldsymbol{p}_n^T]^T \in \mathbb{R}^{dn}$ be the configuration of all the agents. The interaction between the agents is modeled by a graph $\mathcal{G}(\mathcal{V}, \mathcal{E})$ where $\mathcal{V} = \{1,2,...,n\}$ is the node set and $\mathcal{E} \subseteq \mathcal{V} \times \mathcal{V}$ is the edge set. For a digraph, the edge set consists of directed edges $(j, i)$, where node $j$ is the tail node and node $i$ is the head node. Node $j$ is called the in-neighbor of node $i$, while node $i$ is called the out-neighbor of node $j$ and $\mathcal{N}_i:=\{j|(j, i)\in \mathcal{E}\}$ denotes the in-neighbor set for agent $i$. In a digraph, $(j, i) \in \mathcal{E} \neq (i, j)\in\mathcal{E}$. 
   
   A formation, denoted as $(\mathcal{G}, \boldsymbol{r})$, is the graph $\mathcal{G}$ with its vertex $i$ mapped to point $\boldsymbol{r}_i$, where $\boldsymbol{r}$ represents the configuration of $\mathcal{G}$. Without loss of generality, suppose the first $n_l$ agents are leaders and the rest $n_f = n - n_l$ agents are followers. Let $\mathcal{V}_l = {1,...,n_l}$ and $\mathcal{V}_f = \mathcal{V}\setminus \mathcal{V}_l$ be the sets of leaders and followers, respectively. The current positions of the leaders and followers are denoted as $\boldsymbol{p}_l = [\boldsymbol{p}_1^T,...,\boldsymbol{p}_{n_l}^T]^T$ and $\boldsymbol{p_f} = [\boldsymbol{p}_{n_l+1}^T,...,\boldsymbol{p}_n^T ]^T$, respectively. 
   
  \subsection{Affine Image and Affine Localizability}
   For configuration $\boldsymbol{r}$, if there are real matrixs $\bold{A}$ and $\bold{B}$ with appropriate dimensions, $\mathcal{A}(\boldsymbol{r}) = \{\boldsymbol p \in \mathbb{R}^{dn} :=(\boldsymbol{I}_n \otimes \bold{A})\boldsymbol{r} + \bold{1} \otimes \boldsymbol{B}\}$, $\mathcal{A}(\boldsymbol{r})$ is defined as an \textit{Affine Image} of $\boldsymbol{p}$. Given that the formation configuration $\boldsymbol{r}$, then at any time during the affine formation movement, the formation $\boldsymbol{p}$ must satisfy that $\boldsymbol{p}\in \mathcal{A}(\boldsymbol{r})$.
   
   In a hierarchical formation, it is necessary to infer the current position of the followers based on the position of the leader. According to \cite{Zhao2018}, \textit{affine localizability} can be defined as follows:
   For any $\boldsymbol{p}=[\boldsymbol{p}_l, \boldsymbol{p}_f]^T\in \mathcal{A}(r)$, if $\boldsymbol{p}_f$ can be uniquely determined by $\boldsymbol{p}_l$,  then the formation is affinely localizable.
   
   The conditions for \textit{affine localizability} can be divided into position and stress conditions. Given a set of points $\boldsymbol{p}$, let $\bar{\boldsymbol{p}} \in \mathbb{R}^{n\times(d+1)}=[\boldsymbol p, \bold{1}_n]$. If and only if $n\geq d+1$ and $rank(\bar{\boldsymbol p})=d+1$, $\boldsymbol p$ \textit{affinely span} $\mathbb{R}^d$. For the formation $(\mathcal{G}, \boldsymbol{r})$, assume that $\boldsymbol{r}$ \textit{affinely span} $\mathbb{R}^d$. The position condition for \textit{affine localizability} is as follows:
   \begin{lemma}
   \label{lemma1}
    The formation $(\mathcal{G}, \boldsymbol r)$ if \textit{affinely localizability} if and only if $\boldsymbol r$ \textit{affinely span} $\mathbb{R}^d$.
   \end{lemma}

  \subsection{Stress Matrices and Affine Maneuver}
   For formation $(\mathcal{G}, \boldsymbol{r})$, a stress is denoted as $\{\omega_i^j\}_{(i, j\in \mathcal{E})}$. According to Ref. \cite{Connelly2005}, a stress is called an equilibrium stress if it satisfies
   \begin{equation}
   \label{e1}
    \sum_{j\in \mathcal{N}_i}\omega_i^j(\boldsymbol{p}_j-\boldsymbol{p}_i)=0,\quad i\in \mathcal{V}
   \end{equation}
   To simplify the expression of Eq.(\ref{e1}), the stress matrix is defined as $\mathbf\Omega \in \mathbb{R}^{n\times n}$ which satisfies
   \begin{equation}
   \label{e2}
    \left\{\begin{array}{ll}
    0,&i\neq j,(i,j)\notin \mathcal{E}\\
    -\omega_i^j,&i\neq j, (i, j)\notin \mathcal{E}\\
    \sum_{k\in \mathcal{N}_i}\omega_i^k, &i=j
    \end{array}\right.
   \end{equation}
   Then, the Eq.(\ref{e1}) can be expressed in a matrix form as 
   \begin{equation}
   \label{e3}
    (\boldsymbol{\Omega} \otimes \boldsymbol{I}_d)\boldsymbol{p}=\mathbf{0}
   \end{equation}
   In this paper, denote $\boldsymbol{\bar{\Omega}}=\boldsymbol{\Omega} \otimes \boldsymbol{I}_d$ for notational simplicity. Partition $\boldsymbol{\bar{\Omega}}$ according to the partition of leaders and followers as $\boldsymbol{\bar{\Omega}}=\begin{bmatrix}\boldsymbol{\bar{\Omega}}_{ll} & \boldsymbol{\bar{\Omega}}_{lf} \\ \boldsymbol{\bar{\Omega}}_{fl} & \boldsymbol{\bar{\Omega}}_{ff}                                                                                                                                                                                                                                                            \end{bmatrix}$. The stress condition for \textit{affine localizability} is as follows:
   \begin{lemma}
     \label{l1} The nominal formation ($\mathcal{G}, \boldsymbol{r}$) is affinely localizable if and only if $\boldsymbol{\bar \Omega}_{ff}$ is nonsingular. When $\boldsymbol{\bar \Omega}_{ff}$ is nonsingular, for any $\boldsymbol{p}=[\boldsymbol{p}_l^T, \boldsymbol{p}_f^T]^T\in \mathcal{A}(\boldsymbol r)$, $\boldsymbol{p}_f$ can be uniquely calculated as $\boldsymbol{p}_f=-\boldsymbol{\bar \Omega}_{ff}^{-1}\boldsymbol{\bar \Omega}_{fl}\boldsymbol{p}_l$.
    \end{lemma}

 \subsection{Problem Statement} 
  During the process of a formation navigating through any given environment $E$, the formation shape is often constrained by the characteristics of the environment $E$. For a reconfigurable hierarchical formation $(\mathcal{G}^t, \boldsymbol{r}^t)$, dynamic reorganization of the leader layer typically implies different travel advantages under the environment  $E$. These advantages may include the shortest  travel time, optimal formation visibility, the most stable communication, and so on. Building on this, we construct an optimization equation according to the travel advantages associated with formation reorganization:
     
   \begin{figure*} 
    \raggedright
	 \includegraphics[width=1\linewidth]{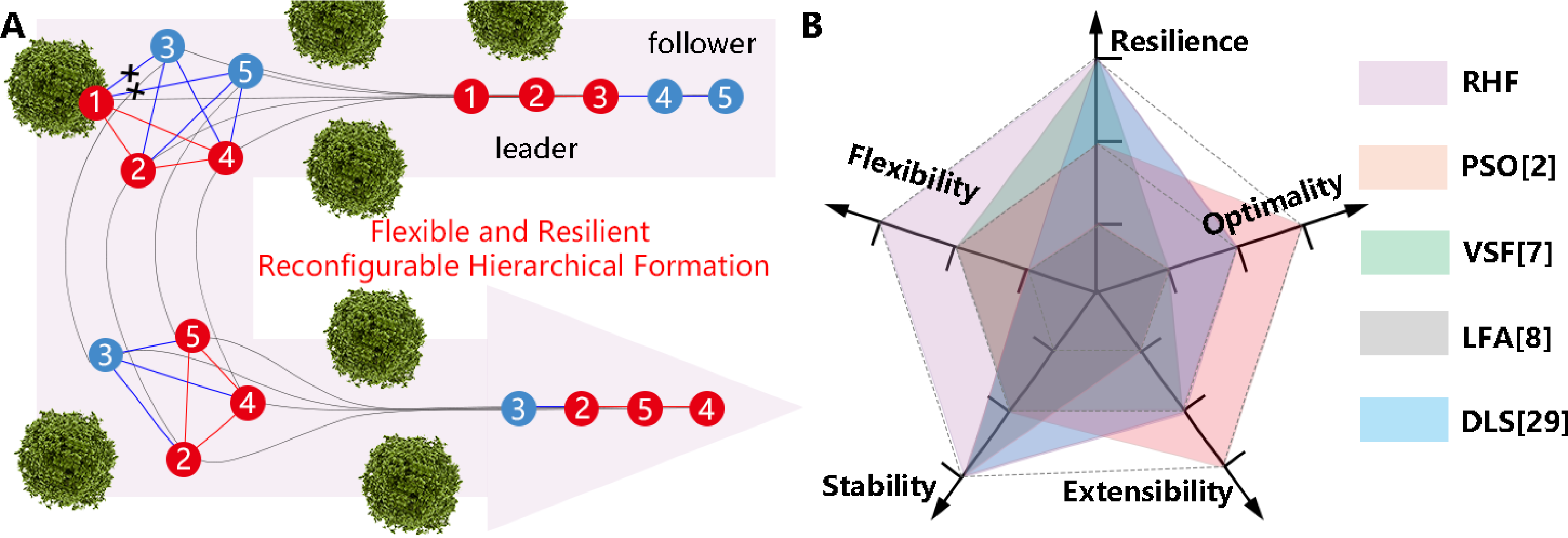} 
        \caption{A.The system architecture of reconfigurable hierarchical formation. B.Comparison of with leader-follower affine formation(LFA)\cite{Zhao2018}, virtual structure formation(VSF)\cite{Zhou2018}, dynamic leader selection(DLS)\cite{Li2017}, particle swarm optimization(PSO)\cite{Zhou2022} and our proposed method RHF. The radar chart shows that moving outward from the center represents increasing corresponding versatile. Specifically, stability represents the ability of the formation to maintain its standard shape during movement. Resilience represents the formation's ability to recover when individual agents lose connectivity. Flexibility represents the formation's adaptability to complex environments. Optimality, as defined in \cite{Zhou2022}, represents the ability to seek optimal formation in spatial, temporal, or other user-defined scales. Moreover, extensibility represents the ability to analyze and model formations for specific tasks.}  
	\label{fig:overview} 
   \end{figure*}
  \begin{equation}
    \begin{aligned}
    J=\phi(\mathcal{G}^t,\boldsymbol r^t,E)\quad\quad\quad\quad\\ s.t.\quad det(\boldsymbol{\Omega}_{ff}^{\mathcal{G}^t})>0\quad\quad \boldsymbol r^t \in \mathcal{A}(\boldsymbol r^o)
    \end{aligned}
  \end{equation}
   where $\phi(\cdot)$ denotes the function that represents different maneuvering advantages under the current environmental conditions $E$. The condition $det(\Omega_{ff})>0$ represents the topology that must be satisfied to allow the positions of the follower layer to be determined by their neighbors. Meanwhile,  $\boldsymbol{r}^t \in \mathcal{A}(\boldsymbol{r}^o)$ specifies the positional constraints to guarantee the feasibility of the formation reorganization. This paper focuses on analyzing the constraints of the proposed optimization equation, aiming to develop a reconfigurable hierarchical formation that is both flexible and resilient.
  
   \section{Feasible Conditions For Reconfigurable Hierarchical Formation}
   This section comprehensively elaborates on the specific methods of reconfigurable hierarchical formation, divided into two distinct parts, including the two parts of hierarchical reorganization. Firstly, we define the reconfigurable hierarchical formation and analyze the necessary conditions for hierarchical reorganizations. To eliminate disturbances caused by other followers, a power-centric topology switching strategy is proposed to facilitate the smooth execution of hierarchical reorganizations. 
   
 \subsection{Reconfigurable Hierarchical Formation}
    \begin{figure}
    \raggedright
	 \includegraphics[width=1\linewidth]{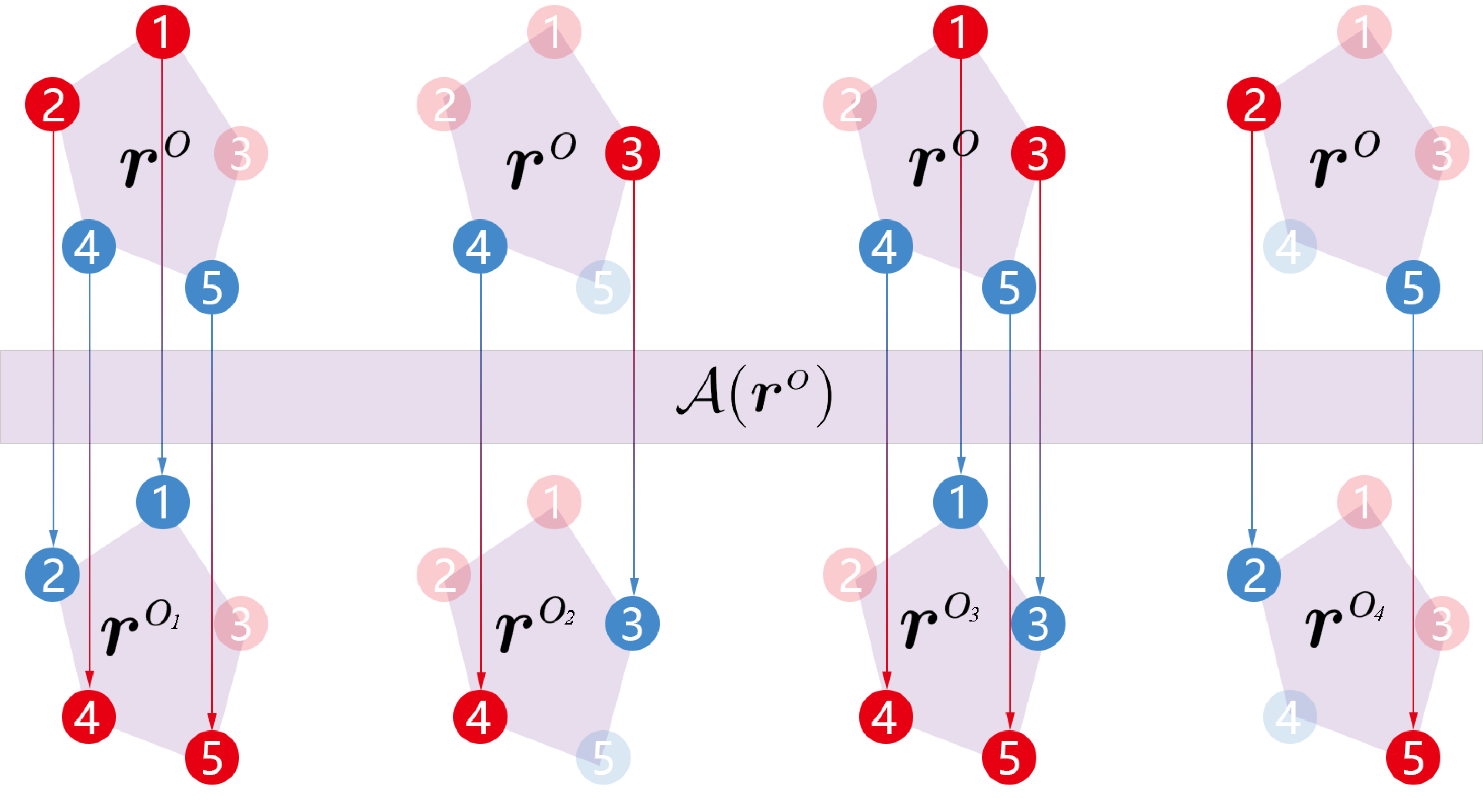} 
	 \caption{Illustration of Feasible Configurations for Reconfigurable Hierarchical Formation.}  
	\label{fig:dynamic_configuration} 
   \end{figure} 
   
   \textit{Dynamic affine localizability}, designed to accommodate hierarchical reconfiguration of the formation, means that the roles of agents within the formation can switch between leader and follower. Let the changing graph $\mathcal{G}^t=\mathcal{G}(\mathcal{V}(t), \mathcal{E}(t))$, we define \textit{dynamic affine localization} as:

    \begin{myDef}
     \label{Def1}(Dynamic Affine Localizability) For the changing formation $(\mathcal{G}^t, \boldsymbol r^t)$, the position of followers $\{\boldsymbol r_i\}_{i\in \mathcal{V}_f(t)}$ can be uniquely determined by the position of leaders $\{\boldsymbol r_i\}_{i\in \mathcal{V}_l(t)}$.
    \end{myDef}
    
    \begin{myAum}
     \label{Aum1}(Initial Formation for Dynamic Affine Localizability): For the initial formation $(\mathcal{G}^o, \boldsymbol r^o)$, assume that $\{\boldsymbol r^o_i\}_{i\in \mathcal{V}^o_l}$ affinely span $\mathbb{R}^d$.
    \end{myAum}
    
    \begin{theorem}
     \label{theo1}(Hierarchical Reorganization for Dynamic Affine Localizability): Under Assumption \ref{Aum1}, the changing formation $(\mathcal{G}^t, \boldsymbol r^t)$ is dynamically affine localizable if and only if $\{\boldsymbol r^t_i\}_{i\in \mathcal{V}^t_l}$ affinely span $\mathbb{R}^d$.
    \end{theorem}
    
    \begin{proof}(Sufficiency)
    According to Lemma \ref{lemma1}, the formation $(\mathcal{G}, \boldsymbol r)$ is \textit{affinely localizable} if and only if $\{\boldsymbol r_i\}_{i\in \mathcal{V}_l}$ affinely span $\mathbb{R}^d$. For any formation $(\mathcal{G}^{t_1}, \boldsymbol r^{t_1})\in (\mathcal{G}^{t}, \boldsymbol r^{t})$, $(\mathcal{G}^{t_1}, \boldsymbol r^{t_1})$ is \textit{affinely localizable} if $\{\boldsymbol r_i\}_{i\in \mathcal{V}_l}^{t_1}$ affinely span $\mathbb{R}^d$. Then, $(\mathcal{G}^t, \boldsymbol r^t)$ is \textit{dynamically affine localizable} consequently.
    
    (Necessity) For any formation $(\mathcal{G}^{t_1}, \boldsymbol r^{t_1})\in (\mathcal{G}^{t}, \boldsymbol r^{t})$, $(\mathcal{G}^{t_1}, \boldsymbol r^{t_1})$ cannot realize \textit{affine localizability} if $\{\boldsymbol r_i\}_{i\in \mathcal{V}_l}^{t_1}$ cannot \textit{affinely span} $\mathbb{R}^d$. Then, $(\mathcal{G}^t, \boldsymbol r^t)$ cannot satisfy \textit{dynamically affine localizability} consequently.
    \end{proof}
    
    Formations can better leverage the advantages of the group through hierarchical reorganization. We call formations with this characteristic a \textit{reconfigurable hierarchical formation}, which is defined as:
    
    \begin{myDef}
     \label{Def2}(Reconfigurable hierarchical formation) For the changing  formation $(\mathcal{G}^t, \boldsymbol r^t)$, any target configuration $\boldsymbol p^*$ of current nominal configuration $\boldsymbol r^t$ locates at the same affine image $\mathcal{A}(\boldsymbol r^t)$.
    \end{myDef}
    
    \begin{myAum} 
     \label{Aum2}(Initial Configuration for Versatile  Affine Formation) For any initial configuration $\boldsymbol r^o$, assume that $C(n, n_l) \geq 1$ where $C(n, n_l)$ counts the number of viable combinations for hierarchical reorganization.
    \end{myAum}
    
    \begin{theorem}
     \label{theo2} (Reconfigurable Affine Formable) Under Assumptions \ref{Aum1} and \ref{Aum2}, the dynamic configuration $\boldsymbol r^t$ is reconfigurable affine formable if and only if the initial configuration $\boldsymbol r^o$ satisfies that $\{\boldsymbol r^{o_j}\}_{j\in [1, C(n, n_l)]}\in \mathcal{A}(\boldsymbol r^o)$.
    \end{theorem}
    
    \begin{proof}(Sufficiency)
     if $\{\boldsymbol r^{o_j}\}_{j\in [1, C(n, n_l)]}\in \mathcal{A}(\boldsymbol r^o)$, there always exists $(\boldsymbol A', b')$ satisfying 
     \begin{equation}
     \label{e9}
      \boldsymbol r^{o_j} = (\boldsymbol I_n \otimes \boldsymbol A')\boldsymbol r^o + \boldsymbol 1_n \otimes b'
     \end{equation}
    Based on the property of affine transformation, the matrix $\boldsymbol A$ is invertible \cite{Lin2016}, which means the inverse of $\boldsymbol I_n \otimes \boldsymbol A'$ exists. Then, 
    \begin{equation}
    \label{e10}
     \boldsymbol r^o = (\boldsymbol I_n \otimes \boldsymbol A^{'-1})\boldsymbol r^{o_j} - \boldsymbol 1_n \otimes (\boldsymbol A^{'-1}\cdot b')
    \end{equation}
     Assume that the configuration $\boldsymbol p \in \mathcal{A}(\boldsymbol r^o)$ represents the pre-formation in the first leader switching, then there always exists $(\boldsymbol A^m, b^m)$ satisfying 
     \begin{equation}
     \label{e11}
      \boldsymbol p = (\boldsymbol I_n \otimes \boldsymbol A^m)\boldsymbol r^o + \boldsymbol 1_n \otimes b^m
     \end{equation}
    Substituting $\boldsymbol r^o$ with $(\boldsymbol I_n \otimes \boldsymbol A^{'-1})\boldsymbol r^{o_j} - \boldsymbol 1_n \otimes (\boldsymbol A^{'-1}\cdot b')$, then $\boldsymbol p$ can be reformulated as 
    \begin{equation}
    \label{e12}
     \boldsymbol p = (\boldsymbol I_n \otimes \boldsymbol A^m \boldsymbol A^{'-1})\boldsymbol r^{o_j} - \boldsymbol 1_n\otimes (b^m - \boldsymbol A^m \boldsymbol A^{'-1}b')
    \end{equation}
    Make $\boldsymbol A^{m'} =  \boldsymbol A^m \boldsymbol A^{'-1}$ and $b^{m'} = b^m - \boldsymbol A^m \boldsymbol A^{'-1}b'$, then 
    \begin{equation}
    \label{e13}
     \boldsymbol p' = \boldsymbol p = (\boldsymbol I_n \otimes \boldsymbol A^{m'})\boldsymbol r^{o_j} + \boldsymbol 1_n\otimes b^{m'}
    \end{equation}
    where $\boldsymbol p'$ represents the post-formation after switching leader. Since $\boldsymbol p'\in \mathcal{A}(\boldsymbol r^{o_j})$ and $\mathcal{A}(\boldsymbol r^o)=\mathcal{A}(\boldsymbol r^{o_j})$, the dynamic formation $(\mathcal{G}^t,\boldsymbol r^t)$ is affinely formable consequently.

    (Necessity) Taking the first leader switching as an example, there always exists $(\boldsymbol A^m, b^m)$ satisfying 
    \begin{equation}
    \label{e14}
     \boldsymbol p = (\boldsymbol I_n \otimes \boldsymbol A^m)\boldsymbol r^o + \boldsymbol 1_n \otimes b^m
    \end{equation}
    where $\boldsymbol p \in \mathcal{A}(\boldsymbol r^o)$ represents the pre-formation. Since $\boldsymbol p' = \boldsymbol{Mp}$ where $\boldsymbol p'$ and $\boldsymbol M$ represent the post-formation and the transfer matrix, respectively, then 
    \begin{equation}
    \label{e15}
     \boldsymbol p' = (\boldsymbol I_n\otimes \boldsymbol A')\boldsymbol{Mp} + \boldsymbol 1_n\otimes b'
    \end{equation}
    Substituting $\boldsymbol p$ with $\boldsymbol r^o$, then $\boldsymbol p'$ can be reconstructed as
    \begin{equation}
    \label{e16}
     \boldsymbol p' = (\boldsymbol I_n\otimes \boldsymbol A'\boldsymbol A^m)\boldsymbol{M}\boldsymbol r^o + \boldsymbol 1_n\otimes (\boldsymbol M\boldsymbol A'b^m + b')
    \end{equation}
    Since $\boldsymbol r^{o_j}=\boldsymbol M\boldsymbol r^o$, then 
    \begin{equation}
    \label{e17}
     \boldsymbol p' = (\boldsymbol I_n \otimes \boldsymbol A^{m'})\boldsymbol r^{o_j} + \boldsymbol 1_n\otimes b^{m'}
    \end{equation}
   where $\boldsymbol A^{m'} = (\boldsymbol I_n\otimes \boldsymbol A'\boldsymbol A^m)$ and $b^{m'} = \boldsymbol M\boldsymbol A'b^m + b'$. Then, $\boldsymbol p'\in \mathcal{A}(\boldsymbol r^{o_j})$. If $\mathcal{A}(\boldsymbol r^{o}) \neq \mathcal{A}(\boldsymbol r^{o_j})$, $\boldsymbol p'\notin \mathcal{A}(\boldsymbol r^{o_j})$. The dynamic formation $(\mathcal{G}^t,\boldsymbol r^t)$ cannot satisfy \textit{affine localizability} consequently.
    \end{proof}
    
   For a changing formation $\mathcal{G}^t$ satisfying Theorem \ref{theo1} and \ref{theo2}, we can identify a collection of viable formations from  $\{\boldsymbol r^{o_j}\}_{j\in [1, C(n, n_l)]}\in \mathcal{A}(\boldsymbol r^o)$ represented as $\mathbf{r}$.

   \subsection{Power-Centric Topology Switching}
    \begin{figure}[t!]
    \centering
	  \subfigure[\label{fig:2d_dynamic_graph}]{
      \includegraphics[width=1.0\linewidth]{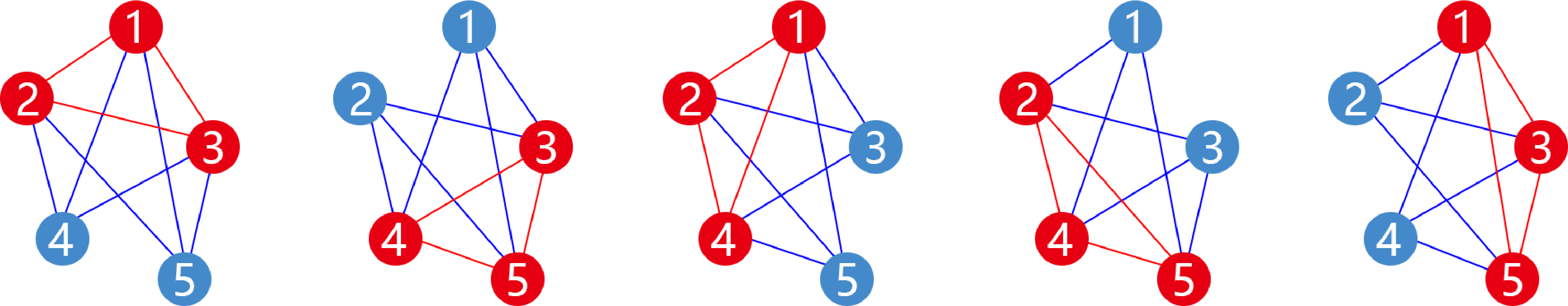}}
      \subfigure[\label{fig:3d_dynamic_graph}]{
      \includegraphics[width=1.0\linewidth]{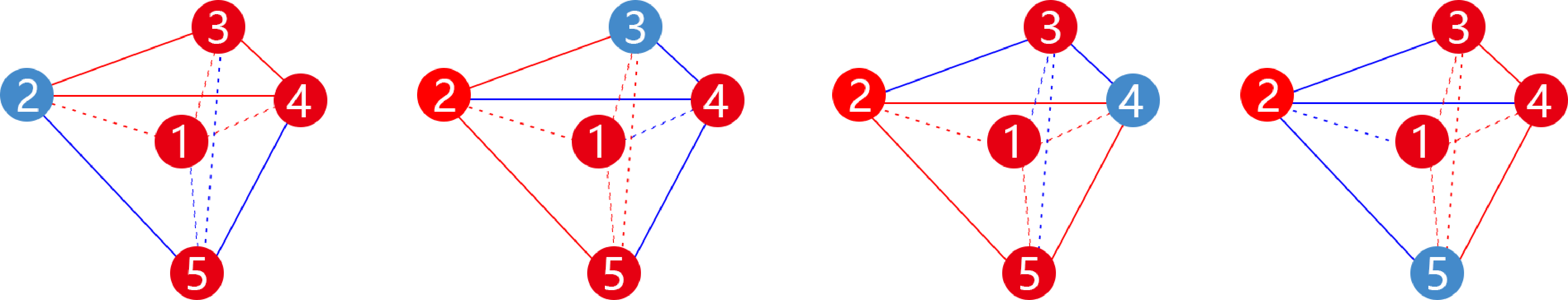}} 
      \label{fig:dynamic_graph} 
    \caption{Examples for 2d and 3d dynamic graph configurations.}
   \end{figure}       
   In hierarchical formation, control under affine transformations typically manifests as leaders and followers employing distinct strategies. The follower's position, inherently dependent on topological relationships, exhibits a coupling effect between the formation's inherent topological structures and the temporal difference, affecting the control accuracy. Considering the dynamic change of individual roles within the hierarchical reorganization, invariant topological connections can compromise control accuracy. 
   
   Firstly, the impact of formation topology on control precision is analyzed. Eq. (\ref{equa1}) presents the classical calculation formula for the follower's position.
   \begin{equation}
   \label{equa1}
    \dot{\boldsymbol x}_i = -\frac{1}{\sum_{j\in N_i}\omega_{ij}}\sum_{j\in N_i}\omega_{ij}(\boldsymbol x_i- \boldsymbol x_j-\dot{\boldsymbol x}_j)
   \end{equation}
   where $x_i$ and $x_j$ are the current positions of agent $i$ and $j$, respectively. Due to the presence of errors between the current position and the target position during actual flight, let $\bar{x}_i$ and $\bar{x}_j$ represent the target positions of i and j, respectively. Then, $\epsilon_i = x_i - \bar{x}_i$ and $\epsilon_j = x_j - \bar{x}_j$. Substituting these into Eq. (\ref{e1}), then
   \begin{equation}
   \label{equa2}
    \sum_{j\in{N_i}}\omega_{ij}(\boldsymbol x_i+\boldsymbol \epsilon_i-\boldsymbol x_j-\boldsymbol \epsilon_j)=0
   \end{equation}
    Furthermore, the stress formulation can be expressed as:
    \begin{equation}
     \boldsymbol \epsilon_i-\frac{\sum_{j\in{N_i}}\omega_{ij} \boldsymbol \epsilon_j}{\sum_{j\in{N_i}}\omega_{ij}}=-\frac{\sum_{j\in{N_i}}\omega_{ij}(\boldsymbol{x}_i-\boldsymbol x_j)}{\sum_{j\in{N_i}}\omega_{ij}}
    \end{equation}
    Based on this, Eq. (\ref{equa1}) can be reformulated as:
    \begin{equation}
    \label{equa3}
     (\boldsymbol{\epsilon_i}-\frac{\sum_{j\in{N_i}}\omega_{ij}{\boldsymbol{\epsilon}_j}}{\sum_{j\in{N_i}}\omega_{ij}})=-(\boldsymbol{\dot{\epsilon}_i}-\frac{\sum_{j\in{N_i}}\omega_{ij}\boldsymbol{\dot{\epsilon}_j}}{\sum_{j\in{N_i}}\omega_{ij}})
    \end{equation}
    
   From Eq. (\ref{equa3}), it can be derived that under the follower position calculation equation Eq. (\ref{equa1}), the taget position error of agent $i$ is related to the time and the position tracking errors of its neighbors. In response, the following part delves into a power-centric topological switching approach designed to mitigate the influence of topological connections in tandem with the hierarchical reorganization process.
   
   Between the factors of time and neighbors' position errors, selecting agents with minimal position errors as neighbors can reduce target position calculation errors. The tracking error for leaders is primarily influenced by control errors, whereas for followers, it is affected by both target position calculation errors and control errors. Assuming consistent control performance between leaders and followers, choosing followers as neighbors rather than leaders amplifies the target position calculation error. Furthermore, this calculation error increases with the number of follower neighbors. Based on this, we propose a feasible topology switching strategy to minimize position calculation errors and demonstrate its feasibility.
   \begin{lemma}
     \label{l3}(Generic Graph Configuration \cite{Lin2016}) For the dynamic formation $\mathcal{G}(\mathcal{V}(t), \mathcal{E}(t))$, if and only if $\mathcal{G}(\mathcal{V}(t), \mathcal{E}(t))$ is $(d+1)-rooted$, the stress matrix $\Omega$ is semi-positive and $rank(\Omega)=n-d-1$.
    \end{lemma}
   
   \begin{myAum}
    \label{a3} For the dynamic formation $\mathcal{G}(\mathcal{V}(t), \mathcal{E}(t))$, $\mathcal{G}(\mathcal{V}(t), \mathcal{E}(t))$ is $(d+1)-rooted$.
   \end{myAum}
   
   \begin{theorem}
    \label{t3} (Topology Configuration for Reconfigurable Affine Localizability) Under the Assumption \ref{a3}, for the dynamic formation $\mathcal{G}(\mathcal{V}(t), \mathcal{E}(t))$, when each follower is topologically connected only to leaders, $\mathcal{G}(\mathcal{V}(t), \mathcal{E}(t))$ can achieve reconfigurable affine localizability.
   \end{theorem}
   
   \begin{proof}
    When $\mathcal{G}(\mathcal{V}(t),\mathcal{E}(t))$ is $(d+1)-rooted$, there exists at least $(d+1)-neighboor$ for each follower. As each follower is topologically connected only to leaders, each follower will be topologically connected to at least $(d+1)$ leaders. Under this circumstance, the adjacent matrix $\boldsymbol{D}$ can be reconstructed as $\boldsymbol{D}=\begin{bmatrix}\boldsymbol{D_{ll}} & \boldsymbol{D_{lf}}\\ \boldsymbol{0} & \boldsymbol{D_{fl}}                                                                                                                                                                                                                                                       \end{bmatrix}$. Denote the total number of edges between the follower $i$ and leaders as $n_{ef}^{i}$. As there is no topological connection between followers, then $\boldsymbol{D}_{fl}$ can be expressed as:\\
    \begin{equation*}
    \label{e18}
    \small
     \boldsymbol{D}_{fl}=\begin{bmatrix}-\mathbf{1}\in\mathbb{R}^{1\times n_{ef}^1}&\mathbf{0}\in\mathbb{R}^{1\times n_{ef}^2} & ... & \mathbf{0}\in\mathbb{R}^{1\times n_{ef}^{n_f}}\\ \mathbf{0}\in\mathbb{R}^{1\times n_{ef}^1} & -\mathbf{1}\in\mathbb{R}^{1\times n_{ef}^2} & ... & \mathbf{0}\in\mathbb{R}^{1\times n_{ef}^{n_f}}\\ ... & ... & ... & ... \\ \mathbf{0}\in\mathbb{R}^{1\times n_{ef}^1} & \mathbf{0}\in\mathbb{R}^{1\times n_{ef}^2} & ... & -\mathbf{1}\in\mathbb{R}^{1\times n_{ef}^{n_f}} \end{bmatrix}
    \end{equation*}
    where $\mathbf{1}$ and $\mathbf{0}$ implies all-1 and all-0 vector, respectively. As $\boldsymbol{\Omega}=\boldsymbol{D}diag(\omega)\boldsymbol{D}^T$, then,
    \begin{equation}
    \label{e19}
    \boldsymbol{\omega}_{ef}=\{\omega_1^1,...,\omega_1^{n_{ef}^1},...,\omega_{n_f}^1,...,\omega_{n_f}^{n_{ef}^{n_f}}\}.
    \end{equation}
    The stress block $\boldsymbol{\Omega}_{ff}$ can be constructed as: 
    \begin{equation}
    \label{e20}
     \boldsymbol{\Omega}_{ff}=\boldsymbol{D}_{fl}diag(\boldsymbol{\omega}_{ef})\boldsymbol{D}_{fl}^T=\bigoplus_{i=1}^{n_{f}}\sum_{j=1}^{n_{ef}^{i}}\omega_i^j
    \end{equation}
    According to the stress equilibrium condition, for each follower $i$, $\sum_{j=1}^{n_{ef}^i}\omega_i^j(\boldsymbol{r}_i-\boldsymbol{r}_j)=0$, where $\boldsymbol{r}_i$ and $\boldsymbol{r}_j$ represent the initial configuration of follower and leader, respectively. With this prerequisite, we can obtain that $\sum_{j=1}^{n_{ef}^i}\omega_i^j\neq0$, which is proved in Proposition \ref{p1}. Under this condition, 
    \begin{equation}
    \label{e21}
     \det(\boldsymbol{\Omega}_{ff})=\vert\prod_{i=1}^{n_f}(\sum_{j=1}^{n_{ef}^{i}}\omega_i^j)\vert>0
    \end{equation}
    Based on the characteristics of the matrix det, $\boldsymbol{\Omega}_{ff}$ is nonsingular, which means $\mathcal{G}(\mathcal{V}(t), \mathcal{E}(t))$ can achieve \textit{affine localizability}.
   \end{proof}
    
   \begin{myPro}
    \label{p1} (Stress characteristics for followers) For a dynamic formation $\mathcal{G}(\mathcal{V}(t), \mathcal{E}(t))$ that satisfies Assumption \ref{a3}, when each follower is topologically connected only to the leader, $\sum_{j=1}^{n_{ef}^i}\omega_i^j\neq0$, where $n_{ef}^{i}$ denotes the total number of edges between follower $i$ and leaders.
   \end{myPro}
  
   \begin{proof}
    (contradiction) For follower $i$, assume $\sum_{j=1}^{n_{ef}^i}\omega_i^j=0$, then $\omega_i^{n_{ef}^i}=-\sum_{j=1}^{n_{ef}^i-1}\omega_i^j$. According to the stress equilibrium condition, for each follower $i$, $\sum_{j=1}^{n_{ef}^i}\omega_i^j(\boldsymbol{r}_i-\boldsymbol{r}_j)=0$, where $\boldsymbol{r}_i$ and $\boldsymbol{r}_j$ represent the initial configuration of follower and leader, respectively. From this, we can obtain that:
    \begin{equation}
    \label{e22}
     \sum_{j=1}^{n_{ef}^i-1}\omega_i^j(\boldsymbol{r}_i-\boldsymbol{r}_j)-\sum_{j=1}^{n_{ef}^i-1}\omega_i^j(\boldsymbol{r}_i-\boldsymbol{r}_{n_{ef}^i})=0
    \end{equation}
    Simplification of the above equation yields that:
    \begin{equation}
    \label{e23}
    \sum_{j=1}^{n_{ef}^i-1}\omega_i^j\boldsymbol{r}_j-\sum_{j=1}^{n_{ef}^i-1}\omega_i^j\boldsymbol{r}_{n_{ef}^i}=0
    \end{equation}
    As the topological connectivity exists only between follower $i$ and leaders, then, $n_{ef}^i\leq n_l$. Complementing the non-existent $\boldsymbol{r}_j$ in Eq.(\ref{e23}) with 0, then Eq.(\ref{e23}) can be reconstructed as:
    \begin{equation}
    \label{e24}
     \sum_{j=1}^{n_l-1}\omega_i^j\boldsymbol{r}_j-\sum_{j=1}^{n_l-1}\omega_i^j\boldsymbol{r}_{n_l}=0
    \end{equation}
    Denote $\boldsymbol{\bar{r}}=\begin{bmatrix}\boldsymbol{r}_1^T & \boldsymbol{r}_2^T & ... & \boldsymbol{r}_{n_l}^T \\ \mathbf{1} & \mathbf{1} & ... & \mathbf{1} \end{bmatrix}^T$, then $\boldsymbol{\bar{r}}^T\boldsymbol{\omega}=0$. However, according to the Lemma 1 in \cite{Zhao2018}, $rank(\boldsymbol{\bar{r}})=d+1$, which implies $\boldsymbol{\bar{r}}^T$ is of row-full rank. Then, from the properties of vector spaces, only zero vectors exist in the zero space of $\boldsymbol{\bar{r}}^T$. According to the definition of $\omega$ in Eq.(\ref{e2}), $\sum_{j=1}^{n_{ef}^i}\vert\omega_i^j\vert>0$, which contradicts the above  inference obviously.
   \end{proof}
   According to Theorem \ref{t3}, when the hierarchical reorganization is completed through Section \MakeUppercase{\romannumeral4}-A, the topological connection structure in the current configuration is reset, such that each follower establishes a topological connection only with leaders. An example of topological switch is shown in Fig. \ref{fig:dynamic_graph}.
    
   \begin{figure} 
    \raggedright
	 \includegraphics[width=1\linewidth]{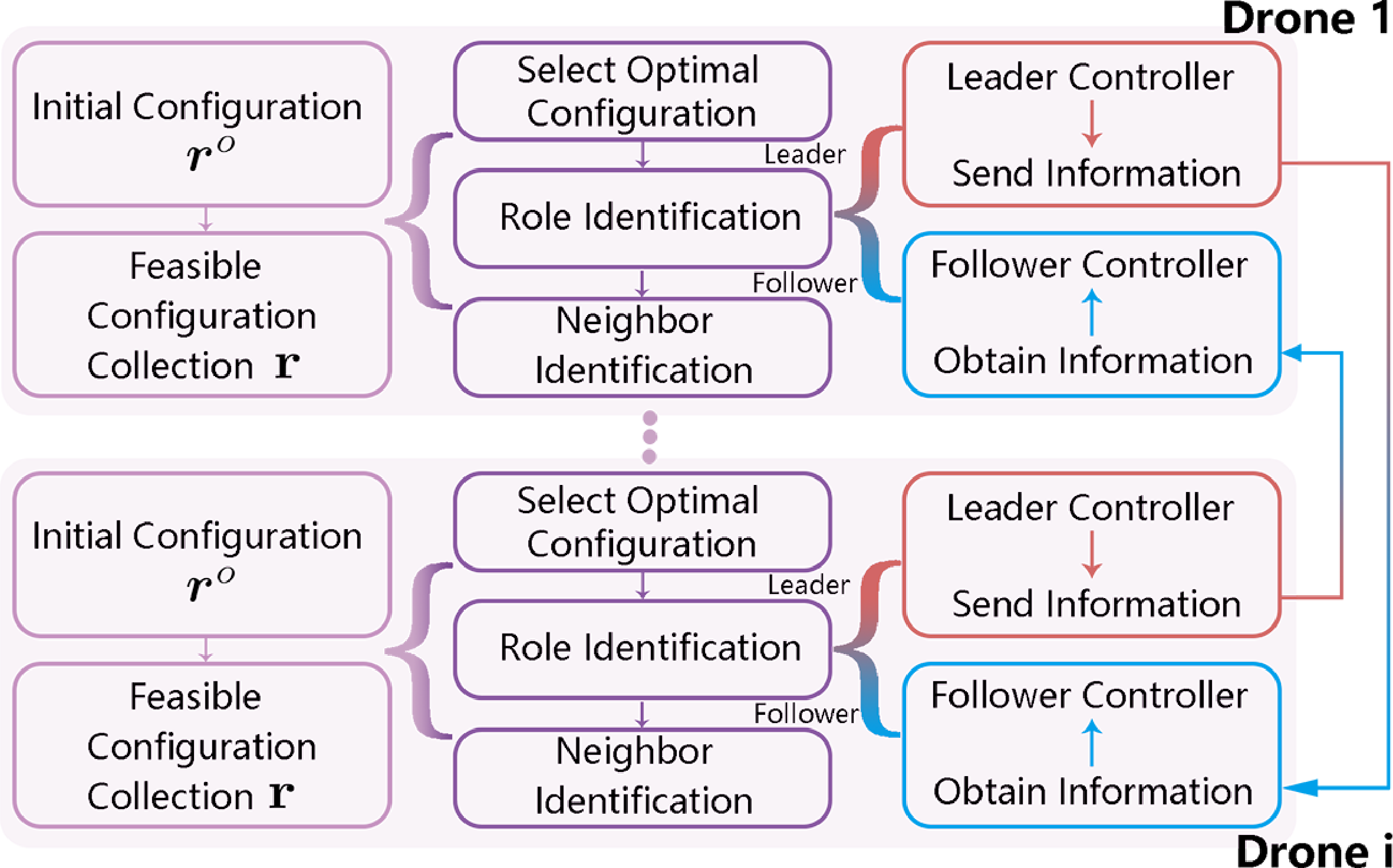}  
        \caption{The system architecture of reconfigurable hierarchical formation.}  
	\label{fig:framework} 
   \end{figure}

  \section{Experimental Verification}
   \subsection{Simulation}

    The simulation test is carried out in the robot operating system (ROS) + Gazebo + PX4. Fig. \ref{fig:3d_sim} shows the formation's overall movement process and trajectory. Five target points are set up during the simulation, and the formation performs a total of $5$ hierarchical reorganizations. 
    \begin{figure}[t!]
        \centering
         \includegraphics[width=1.0\linewidth]{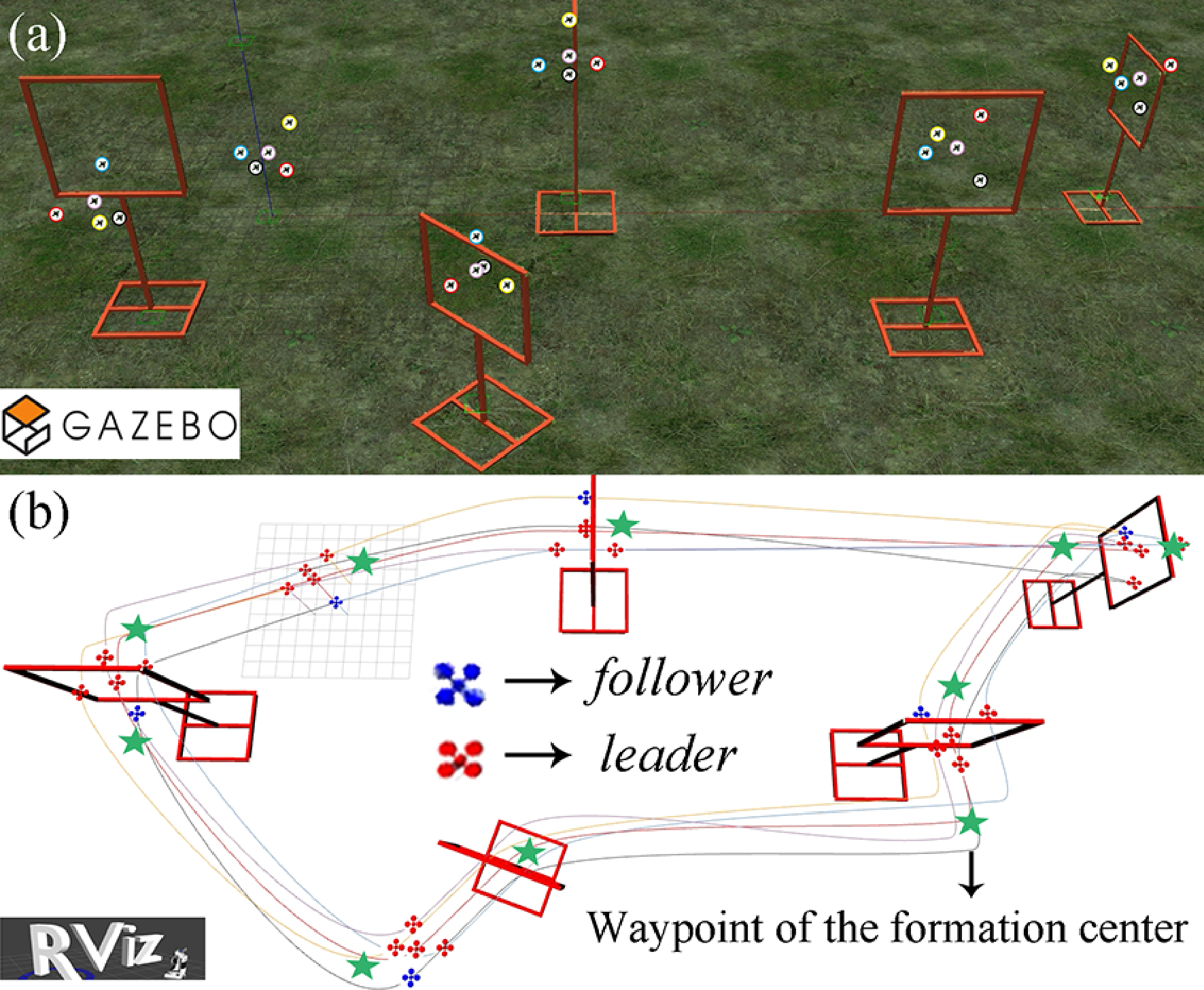} 
        \caption{Simulation Process in 3D space.(a) Simulation in Gazebo. (b) Visulization in RViz.} 
        \label{fig:3d_sim}    
     \end{figure}
     
    Table \ref{tab:path_comparison} illustrates the impact of hierarchical reorganizations on the different drone path lengths. When employing the formation strategies, Drone 2 exhibits the longest motion path, measuring $121.92m$, while Drone 4 covers the shortest distance at $120.38m$, resulting in a difference of $1.54m$. Conversely, without hierarchical reorganizations, Drone 1 records the longest motion path within the formation at $126.92m$, while Drone 3 covers the shortest distance at $120.88m$, yielding a substantial difference of $6.08m$. This signifies that the proposed strategies can effectively equalize the motion distances of drones, thereby significantly enhancing both the coordination within the formation and the formation's overall endurance.
    
    \begin{table}
    \small
    \caption{Comparison of different drones' path length with or without hierarchical reorganizations}
    \label{tab:path_comparison}
    \centering
    \begin{spacing}{1}
    \begin{threeparttable}
    \resizebox{\linewidth}{!}{
    \begin{tabular}{cccccc}
    \toprule[1.5pt]
    Path Length($m$)&Drone 0&Drone 1&Drone 2&Drone 3&Drone 4\\\midrule[1pt]
    With hierarchical reorganization&121.33&121.76&\bf{120.38}&\bf{121.92}&121.76\\
    W/O hierarchical reorganization\cite{Zhao2018}&122.28&\bf{120.88}&\bf{126.92}&122.60&120.43\\
        \bottomrule[1.2pt]
    \end{tabular}}
    \end{threeparttable}
    \end{spacing}
    \end{table}
    
    \begin{table}
    \small
    \caption{Comparison of different drones' velocity with or without hierarchical reorganizations}
    \label{tab:velocity_comparison}
    \centering
    \begin{spacing}{1}
    \begin{threeparttable}
    \resizebox{\linewidth}{!}{
    \begin{tabular}{cccccc}
    \toprule[1.5pt]
    Velocity($m/s$)&Drone 0&Drone 1&Drone 2&Drone 3&Drone 4\\\midrule[1pt]
    With hierarchical reorganization&1.75&1.76&\bf{1.73}&1.75&\bf{1.76}\\W/O hierarchical reorganization\cite{Zhao2018}&1.56&1.54&\bf{1.60}&1.56&\bf{1.53}\\\bottomrule[1.2pt]
    \end{tabular}}
    \end{threeparttable}
    \end{spacing}
    \end{table}

    Table \ref{tab:velocity_comparison} presents the influence of hierarchical reorganizations on the average motion speeds of different drones. When employing the proposed strategy, Drone 4 achieves the highest motion speed, recording a speed of $1.76m/s$, while Drone 2 exhibits the slowest motion speed at $1.73m/s$, resulting in a difference of $0.03m/s$. In contrast, without the hierarchical reorganizations, Drone 4 attains the highest motion speed within the cluster at $1.60m/s$, while Drone 2 moves at the slowest speed of $1.53m/s$, resulting in a more substantial difference of $0.07/s$. Furthermore, the overall motion durations for both approaches are $38.93s$ and $37.31s$, respectively. This indicates that, with similar overall motion durations, the proposed strategy significantly reduces the average speed disparities among individual entities, allowing the swarm to maintain similar motion speeds. Consequently, this enhances both the in-team coordination and flexibility of the formation.

    \subsection{Experimental Setup} 

    To verify the performance of the proposed algorithm, we conduct several indoor experiments, as shown in Fig. \ref{fig:environment}. Specifically, Fig. \ref{fig:environment}(a) depicts a scenario set within an obstacle-free dark environment, where five drones execute formation flights in the $'1'$ and $'8'$ patterns. Fig. \ref{fig:environment}(b) and Fig. \ref{fig:environment}(c) present environments featuring a hoop obstacle and a natural obstacle, respectively. In Fig. \ref{fig:environment}(b), the quintet of drones performs a coordinated obstacle avoidance flight through the hoop, whereas in Fig.  \ref{fig:environment}(c), they engage in formation flights with free transformations within the natural obstacle environment. The dimensions (length$\times$width$\times$height) of each drone are $23cm\times23cm\times14cm$. The IMU is embedded in each drone's flight control hardware, referred to as Pixhawk$^\circledR$. Moreover, the NOKOV$^\circledR$ motion capture system provides the global position measurements, and an UP Board 4000$^\circledR$ computing board running ROS is adopted as the onboard computer. A small TP-LINK$^\circledR$ router model TL-WDR5650 is mounted on each drone to facilitate communication among multiple drones. Each drone is equipped with a light strip, controlled via GPIO on the UP Board 4000$^\circledR$ computing board. This setup allows for visualizing the hierarchical reorganization process within the swarm. When a drone's light strip is illuminated, it indicates that the drone is the current leader of the formation.
     
    \subsection{Results Analysis} 
    
    \begin{figure*}[t!]
        \centering
         \includegraphics[width=1.0\linewidth]{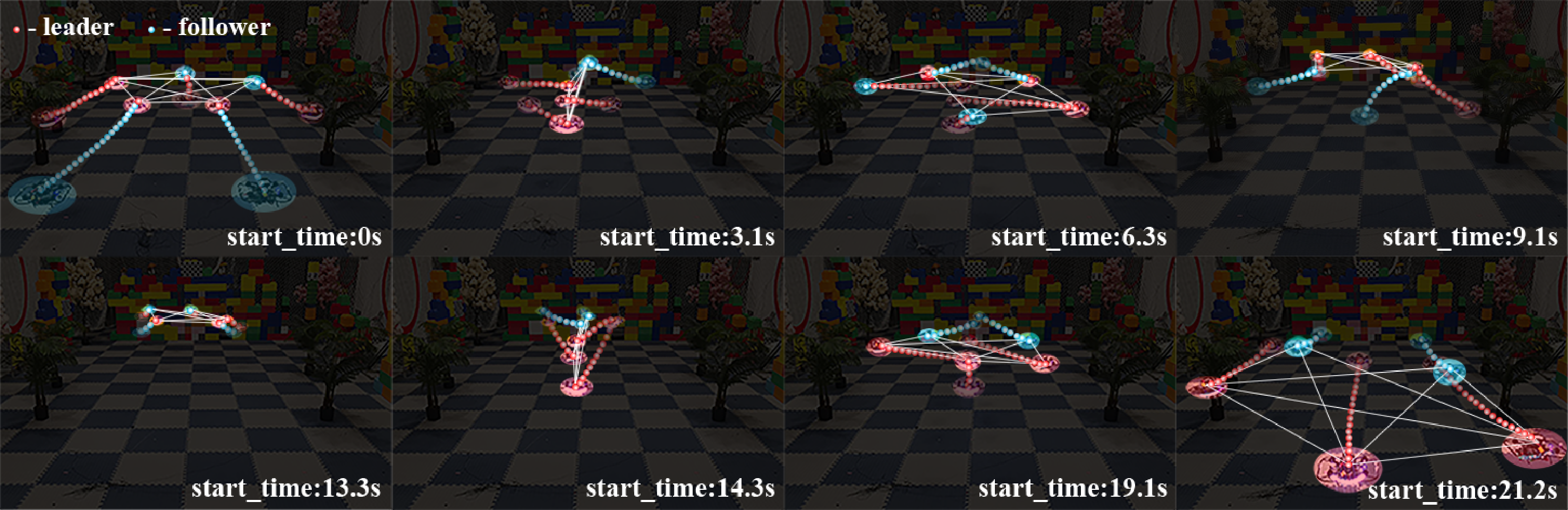}
        \caption{Dense formation transform experiment.}
        \label{fig:transform}  
     \end{figure*} 
     
      \begin{figure}[t!]
        \centering
         \includegraphics[width=1.0\linewidth]{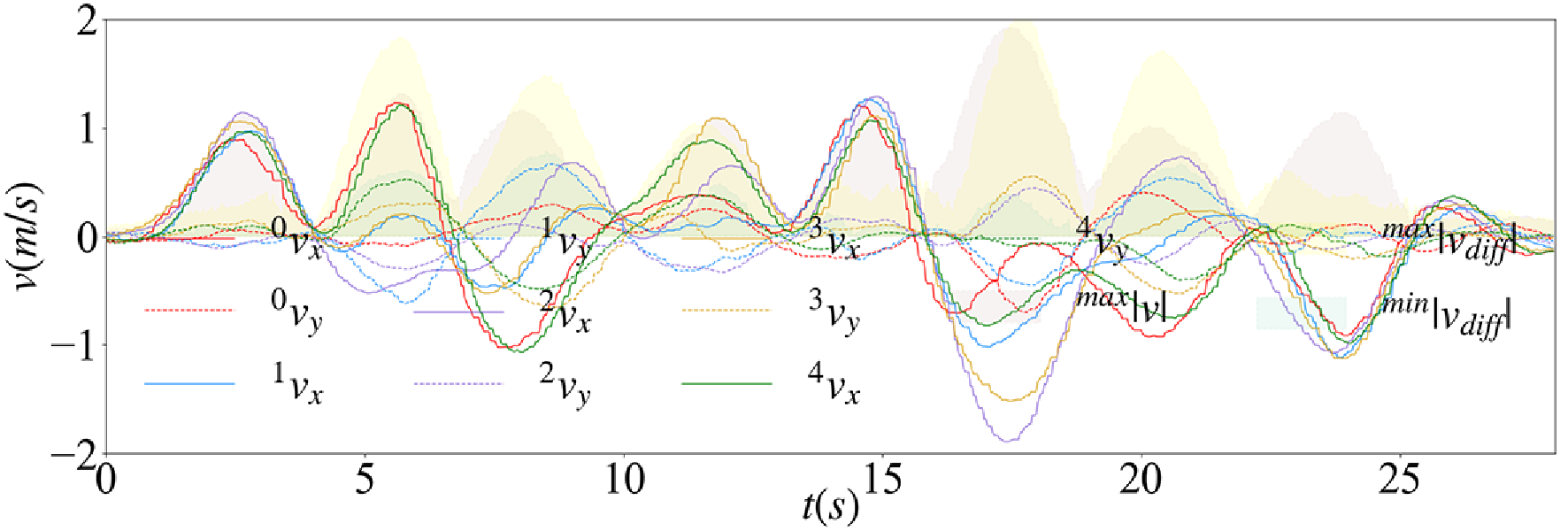}
        \caption{The velocity curve of the formation in Exp. 4.} 
        \label{fig:vel18}    
     \end{figure}
   
        \begin{table}
        \small
        \caption{The hierarchical reorganization scenario and formation transform matrix in Exp. 4}
        \label{tab:exp4_comparison}
        \centering
        \begin{spacing}{1}
        \begin{threeparttable}
        \resizebox{\linewidth}{!}{
        \begin{tabular}{ccccccccc}
        \toprule[1.5pt]
        \multicolumn{9}{c}{Leader Switching}\\
        x $\quad$ & o$\quad$ & x$\quad$ & x$\quad$ & o$\quad$ & x$\quad$ & o$\quad$ &x$\quad$ & x $\quad$\\
        \multicolumn{9}{c}{Transform Matrix}\\
        \multicolumn{9}{c}{$\begin{bmatrix} 1 & 0 \\ 0 & 1 \end{bmatrix}\rightarrow \begin{bmatrix} 0.8 & 0 \\ 0 & 0.8 \end{bmatrix}\rightarrow \begin{bmatrix} 2 & 0 \\ 0.5 & 0 \end{bmatrix}\rightarrow \begin{bmatrix} 1 & 0 \\ 0 & 1 \end{bmatrix}\rightarrow \begin{bmatrix} 1 & 0 \\ 0 & 1 \end{bmatrix}$}\\
        \multicolumn{9}{c}{$\rightarrow \begin{bmatrix} 0.8 & 0 \\ 0 & 0.8 \end{bmatrix} \rightarrow \begin{bmatrix} 2 & 0 \\ 0.5 & 0 \end{bmatrix} \rightarrow \begin{bmatrix} 0.8 & 0 \\ 0 & 0.8 \end{bmatrix} \rightarrow \begin{bmatrix} 1 & 0 \\ 0 & 1 \end{bmatrix}$}\\
        \bottomrule[1.2pt]
        \end{tabular}}
        \end{threeparttable}
        \end{spacing}
    \end{table}

   In the experiment involving dense formation changes, as shown in Fig \ref{fig:transform}, referred to as Exp. 4, five drones are arranged in a regular pentagon with a side length of $1.2m$. The formation undergoes continuous hierarchical reorganizations in the environment depicted in Fig. \ref{fig:environment}(c). The velocity distribution of the formation in Exp. 4 is shown in Fig. \ref{fig:vel18}, with the maximum instantaneous speed within the formation being $1.92m/s$.

   Table \ref{tab:exp4_comparison} reflects the hierarchical reorganization scenario and transformation matrices at times labeled in Subfigures of Fig. \ref{fig:transform}. Combining the information from Fig. \ref{fig:transform} and Table \ref{tab:exp4_comparison}, it can be observed that at $t=3.1s,14.3s$, the formation changes into a straight line, at $t=0s,13.3s,19.1s$, the formation undergoes scaling, at $t=6.3s,21.2s$, it expands, and at $t=3.1s,6.3s,13.3s$, a dynamic leader switch occurs to respond to rapid maneuvers, such as a formation turning. Exp. 4 illustrates that affine transformations under hierarchical reorganization enable flexible transformations and rapid execution of complex maneuvers, such as formation U-turns.  
  
 \section{Conclusion and Discussion}
  This work presents an innovative reconfigurable hierarchical formation approach that leverages hierarchical reorganizations to improve aerial swarms' flexibility and in-team coordination. The efficacy of the proposed methodologies is verified through a series of simulations and extensive real-world trials, which indicate a significant enhancement in team flexibility. Specifically, the formation adeptly executes multiple hierarchical reorganizations, with the most prolonged reconfiguration completed in a mere 0.047s. This rapid adaptability enables a quintet of aerial robots to undertake intricate, collaborative tasks, including executing agile maneuvers and navigating through obstacles at speeds reaching 1.9m/s.
  
  Future works will consider two main points. First, we intend to delve deeper into the control characteristics during the hierarchical reorganization process. This will allow us to enhance the robustness of formation control while maintaining high flexibility. Second, we will seek optimal points for dynamic leader autonomous switching, aiming to increase the autonomy of reconfigurable hierarchical formation.   
 

\printbibliography
%
 \end{document}